\def\year{2017}\relax
\documentclass[letterpaper]{article}
\usepackage{times}
\usepackage{helvet}
\usepackage{courier}
\usepackage{color}
\definecolor{myback}{RGB}{204,232,207}

\usepackage{amsmath}			
\usepackage{amsthm}				
\usepackage{amssymb}			

\usepackage{comment}
\usepackage{algorithm}
\usepackage{algorithmic}

\usepackage{graphicx}
\usepackage{epstopdf}

\makeatletter
\newcommand{\AB}{\mathbf{A}}

\newcommand{\FB}{\mathbf{F}}

\newcommand{\PB}{\mathbf{P}}

\newcommand{\aB}{\mathbf{a}}

\newcommand{\fB}{\mathbf{f}}

\newcommand{\uB}{\mathbf{u}}
\newcommand{\vB}{\mathbf{v}}

\newcommand{\xB}{\mathbf{x}}
\newcommand{\yB}{\mathbf{y}}
\newcommand{\zB}{\mathbf{z}}

\newcommand{\OM}{\mathcal{O}}

\newcommand{\RBB}{\mathbb{R}}

\newcommand{\EBB}{\mathbb{E}}

\newcommand{\argmin}{\mathop{\rm argmin}}

\newcommand{\prox}{\mathrm{prox}}
\newcommand{\defi}{\stackrel{\mathrm{def}}{=}}
\makeatother

\newtheorem{theorem}{Theorem}

\newtheorem{lemma}{Lemma}

\newtheorem{corollary}{Corollary}
\newtheorem{proposition}{Proposition}

\begin{document}
\title{Accelerated Variance Reduced Block Coordinate Descent}
\author{Zebang Shen, Hui Qian, Chao Zhang, and Tengfei Zhou \\
	\{shenzebang,qianhui,zczju,zhoutengfei\_zju@zju.edu.cn\}\\
	Zhejiang University}
\maketitle
\begin{abstract}
	Algorithms with fast convergence, small number of data access, and low per-iteration complexity are particularly favorable in the big data era,
	due to the demand for obtaining \emph{highly accurate solutions} to problems with \emph{a large number of samples} in \emph{ultra-high} dimensional space.
	Existing algorithms lack at least one of these qualities, and thus are inefficient in handling such big data challenge.
	In this paper, we propose a method enjoying all these merits with an accelerated convergence rate $\OM(\frac{1}{k^2})$.
	Empirical studies on large scale datasets with more than one million features are conducted to show the effectiveness of our methods in practice.
\end{abstract}
\section{Introduction}
	In this paper, we consider the minimization of smooth convex function with non-smooth convex regularization:
	\begin{equation} \label{eqn: problem}
	\min_{\xB\in\RBB^d} \FB^\PB(\xB) = \FB(\xB) + \PB(\xB),
	\end{equation}
	where $\FB(\xB) = \frac{1}{n}\sum_{i=1}^n f_i(\xB)$ is the average of $n$ smooth convex component functions $f_i$'s and $\PB(\xB)$ is the possibly non-smooth convex regularization term.
	Many machine learning problems can be phrased as the above problem.
	
	However, the explosive growth of data poses two challenges to solving the aforementioned problem: 
	(i) the number of samples can easily reach the magnitude of millions, 
	and (ii) the dimensionality of these massive datasets is ultra-high in the meantime.
	Fast converging algorithms that meet these challenges have been ardently pursued in the recent years.
	
	To solve optimization problems with enormous samples, the classical Stochastic Gradient Descent (SGD) has gained increasing attention in the last decade.
	The advantage of such methods is that only one sample is access in each iteration.
	While the vanilla version inherently suffers from the slow convergence rate due to the stochastic nature of SGD, one of its variant, named Stochastic Variance Reduced Gradient (SVRG) \cite{johnson2013accelerating}, proposes to mix the exact full gradient and stochastic gradient in a way that better convergence results can be obtained without compromising the low per-iteration sample access \cite{konecny2013semi,mahdavi2013mixed}.
	Works based on such Variance Reduction (VR) technique have proliferated in the past few years.
	For example, \cite{reddi2015variance} extends SVRG to an asynchronous setting so that the parallelism in modern computational architecture can be fully utilized and 
	\cite{allen2015improved} improves the convergence of SVRG in non-strongly convex case and provides a convergence result in non-convex case.
	Besides, there are also alternatives to SVRG, such as SAG~\cite{schmidt2013minimizing} and SAGA~\cite{defazio2014saga}.
	Attempts are made to accelerate the convergence of SGD type methods, e.g. \cite{zhang2015stochastic,frostig2015regularizing,lin2015universal,allen2016katyusha,hien2016accelerated} and the newly proposed Katyusha \cite{allen2016katyusha}, which is the first direct accelerated version of SVRG.
	However, since SGD type methods perform full vector operation in each iteration, they are precluded to handle problems in high dimension.
	
	
	To deal with the ultra-high dimensionality in sparse learning tasks, Coordinate Descent (CD) type methods were given a renewed interest in the past few years.
	The basic idea of Coordinate Descent (CD) type methods is that, in each iteration, approximate the problem with respect to some components of the variable $\xB$ while keeping the rest unchanged \cite{wright2015coordinate}.
	With such technique, full vector operations are avoided, making low per-iteration complexity possible.
	Randomness is also incorporated into CD type methods \cite{nesterov2012efficiency,richtarik2014iteration}, with which convergence results are readily obtained.
	There has been some recent works on CD that introduce non-uniform sampling \cite{qu2014coordinatea,qu2014coordinateb,allen2015even}, and consider the asynchronous~\cite{liu2015asynchronous} and distributed settings~\cite{richtarik2016parallel} for better scalability.
	Accelerated versions of CD type methods have also emerged, for example APPROX \cite{fercoq2015accelerated} gives practical implementation and is applicable to general convex optimization problem.
	A drawback of CD type methods is that all samples are accessed in each iteration.
	When the number of samples is huge, they can still be quite slow.
	
	
%

	As discussed above, existing algorithms only handle problems with either small $n$ or small $d$.
	There has been two exceptions MRBCD \cite{zhao2014accelerated} and S2CD \cite{konecny2014s2cd}, which however have no accelerated version and thus can be further improved.
	In this work, we propose a method called Accelerated Variance Reduced Block Coordinate Descent (AVRBCD) that tackles the two challenges in large scale problem.
	We show that our method enjoys a superior convergence rate.
	Specifically, the advantages of AVRBCD are listed as follows.
	\begin{enumerate}
		\item VR technique is used in AVRBCD, and it accesses only one sample per-iteration in amortized analysis,
		\item AVRBCD avoids full vector operation in each iteration, and is shown to have low per-iteration complexity in solving sparse Empirical Risk Minimization problem,
		\item AVRBCD is an accelerated version of MRBCD, with convergence rate $\OM(\frac{1}{k^2})$, compared to $\OM(\frac{1}{k})$ in MRBCD.
	\end{enumerate}
	To show the effectiveness of AVRBCD in practice, we conduct several experiments on datasets with both large $n$ and large $p$.
	One of our experiment has more than one million variables ($d>10^6$).
	The experiments on the real-word datasets demonstrate superior computational efficiency of our approaches compared to the state of the arts.
\section{Preliminary}
In this section, we give the notations and assumptions used in this paper. The VR, CD, and accelerating techniques are discussed. We also introduce the Empirical Risk Minimization (ERM) problem with linear predictor.
\subsection{Notation \& Assumptions}	
We assume that the variable $\xB \in \RBB^d$ is partitioned into $B$ blocks and B divides $d$ for simplicity.
Let $\Omega = \frac{d}{B}$ be the block size.
The regularization $\PB(\xB)$ is assumed to be separable with respect to the partition of $\xB$, i.e.
\begin{equation}
	\PB(\xB) = \sum_{l=1}^B \PB_l([\xB]_l),
\end{equation}
where $[\xB]_l \in \RBB^\Omega$ corresponds to the $l^{th}$ block of $\xB$.
We shall use $[\xB]_{\backslash l}$ to denote the blocks of $\xB$ other than the $l^{th}$.
Such assumption stands for many important functions, e.g. the sparse inducing $l_1$ norm $\|\cdot\|_1$.
The proximal operator of a convex function $g$ is defined as 
\begin{equation}
	\prox_{g}(\yB) = \argmin_\xB g(\xB) + \frac{1}{2}\|\xB-\yB\|^2,
\end{equation}
where we use $\|\cdot\|$ to denote the Euclidean norm.
We say a function $f$ is $L$-smooth, if for any $\xB$ and $\yB$,
\begin{equation}
	f(\yB) \leq f(\xB) + \langle \nabla f(\xB), \yB -\xB\rangle + \frac{L}{2}\|\xB-\yB\|^2.
\end{equation}
Each component function is assumed to be $f_i$ is $L_i$-smooth, with $L_{max} = \max L_i$.
We also assume that their average $\FB$ is $L_l$-smooth with respect to the $l^{th}$ block and define $L_B = \max L_l$.
Finally, we define $\tilde{L} = \max\{BL_B, L_{max}\}$.
\subsection{Variance Reduction}
SGD uses the gradient of a randomly chosen component function $f_i$ as an unbiased estimator of the exact full gradient $\nabla \FB(\xB_k)$.
The variance introduced by such randomness forces a diminishing step size and leads to slow convergence rate, e.g. $\OM(\frac{1}{\sqrt{k}})$ for smooth convex minimization.
\cite{johnson2013accelerating} proposes to keep the full gradient $\nabla \FB(\tilde{\xB})$ at some snapshot $\tilde{\xB}$ and constructs a mixed gradient as
\begin{equation} \label{eqn: update SVRG}
	\vB_k = \nabla f_i(\xB_k) - \nabla f_i(\tilde{\xB}) + \nabla \FB(\tilde{\xB}).
\end{equation}
They show that the variance of $\vB_k$, i.e. $\|\vB_k - \nabla \FB(\xB_k)\|$, converges to zero when both $\xB_k$ and $\tilde{\xB}$ converge to the optimal point $\xB^*$.
For smooth convex minimization, the best known convergence result of SVRG is $\OM(\frac{1}{k})$ due to \cite{allen2015improved}.
Besides, in doing so, SVRG only accesses one sample per iteration in amortized analysis, making it ideal when only the number of samples $n$ is large.
However, since full vector operation is need in each iteration of such algorithm, SVRG can still be inefficient when the dimensionality $d$ of the problem is ultra-high.
\subsection{Randomized Block Coordinate Descent (RBCD)}
RBCD extends Randomized CD such that a block of coordinates rather than only one coordinate can be updated in each iteration.
More specifically, the update rule writes
\begin{equation} \label{eqn: update RBCD}
	[\xB_k]_l = [\xB_{k-1}]_l - \eta[\nabla \FB(\xB_{k-1})]_l, [\xB_k]_{\backslash l} = [\xB_{k-1}]_{\backslash l},
\end{equation}
where $l$ is the coordinate block randomly selected in the $k^{th}$ iteration and $\eta$ is some step size.
RBCD also bridges CD and Gradient Descent (GD). It becomes CD when the block size $\Omega$ is set to $1$ and when $\Omega$ is set to $d$, it becomes GD.
The convergence rate of RBCD is $\OM(\frac{B}{k})$ \cite{nesterov2012efficiency} which is $B$ times worse than that of GD.
But in applications like sparse ERM problems discussed below, the per-iteration complexity of RBCD can be decreased to $\OM(\frac{d}{B})$, when $B$ is appropriately chosen. 
So CD has the same overall complexity as GD but with a much smaller per-iteration complexity, which is suitable for problems in ultra-high dimensional space.
However, one major flaw about RBCD is that it accesses all $n$ component functions in each iteration. 
When $n$ is large, RBCD can still be slow.
\subsection{APPROX \& Katyusha}
Accelerated versions have been proposed for SVRG and RBCD, namely Katyusha \cite{allen2016katyusha} and APPROX \cite{fercoq2015accelerated} respectively.
Both algorithms utilized the momentum to speed up the convergence.
In Katyusha, the authors use the snapshot $\tilde{\xB}$ as an additional "negative momentum" to overcome the difficulty in handling noisy stochastic gradient.
In APPROX, as originally discussed in \cite{nesterov2012efficiency}, although the convergence is relatively easy to prove, a naive implementation would still involve full vector operation in each iteration.
The authors took the transformation technique proposed in \cite{lee2013efficient} and derived an equivalent alternative to avoid such issue.
Both Katyusha and APPROX can be taken as special case of AVRBCD.
Indeed, when we take $\Omega = d$, AVRBCD degenerates to Katyusha (with opt. II), and when we take $m = 1$ and $\alpha_{3, s} = 0$, AVRBCD becomes APPROX.
\subsection{Empirical Risk Minimization}
 We give a brief introduction to an important class of smooth convex problems, the Empirical Risk Minimization (ERM)  with linear predictor.
 It will be used when analyzing the per-iteration complexity of AVRBCD.
Specifically, we assume each $f_i(\cdot)$ is of the form $f_i(\xB) = \phi_i(\aB_i^\top\xB)$, where $\aB_i$ is the feature vector of the $i^{th}$ sample and $\phi_i(\cdot):\RBB\rightarrow\RBB$ is some smooth convex loss function.
$\AB = [\aB_1 \ldots \aB_n]^\top$ is the data matrix.
In real applications, $\AB$ is usually very sparse and we define to be the sparsity of $\AB$
\begin{equation}
	\rho = \frac{nnz(\AB)}{nd}.
\end{equation}
For simplicity, we assume the zeros in $\AB$ are shattered uniformly.

\section{Methodology}
In this section, we present the proposed algorithm AVRBCD in two different but equivalent forms listed in Algorithm \ref{alg: AVRBCD I} and \ref{alg: AVRBCD II}.
The former is easier to analyze and latter is of more practical interest as it has smaller per-iteration complexity.
\subsection{AVRBCD}
\begin{algorithm}[t]
	\caption{AVRBCD I}
	\begin{algorithmic}[1]
		\label{alg: AVRBCD I}
		\REQUIRE $m, \xB_0, \alpha_{1,0}, \alpha_{2, 0}$
		\STATE $\zB_0 \leftarrow \xB_0$, $\tilde{\xB}^0 \leftarrow \xB_0$;
		\FOR{$s \leftarrow 0$ \TO $S$}
		\STATE $\alpha_{2,s} = \frac{\sqrt{\alpha_{2, s-1}^4 + 4\alpha_{2, s-1}^2} - \alpha_{2, s-1}^2}{2}$;
		\STATE $\alpha_{1, s} = \alpha_{1, s-1}(1-\alpha_{2, s}), \alpha_{3, s} = 1 - \alpha_{1, s} - \alpha_{2, s}$; 
		\STATE $\bar{L}_s = \frac{L_Q}{B\alpha_{3, s}} + L_B, \eta_s = \frac{1}{\bar{L}_s\alpha_{2,s}B}$;
		\STATE $\mu^s = \nabla f(\tilde{\xB}^s)$;
		\FOR{$j \leftarrow 1$ \TO $m$}
		\STATE $k \leftarrow (sm)+j$;
		\STATE $\yB_k = \alpha_{1, s} \xB_{k-1} + \alpha_{2,s} \zB_{k-1} + \alpha_{3, s} \tilde{\xB}^s$; \label{eqn: coupling I}
		\STATE sample $i$ from $\{1, \ldots, n\}$ and $l$ from $\{1, \ldots, B\}$;
		\STATE $\vB_k = \mu^s + \nabla f_i(\yB_k) - \nabla f_i(\tilde{\xB}^s)$; \label{eqn: VR gradient}
		\STATE $[\zB_k]_l = \prox_{\eta_s\PB_i}([\zB_{k-1} - \eta_s \vB_k]_l)$;
		\STATE $[\zB_k]_{\backslash l} = [\zB_{k-1}]_{\backslash l}$;
		\STATE $\xB_k = \yB_k + \alpha_{2, s}B(\zB_k - \zB_{k-1})$; \label{eqn: coupling II}
		\ENDFOR
		\STATE Sample $\sigma_s$ from $\{1, \ldots, m\}$ uniformly;
		\STATE $\tilde{\xB}^{s+1} = \xB_{sm+\sigma_s}$;
		\ENDFOR
	\end{algorithmic}
\end{algorithm}
\noindent We divide our algorithm into epochs, at the beginning of which the full gradient in some snapshot $\tilde{\xB}^s$ is computed. 
As a mixture of (\ref{eqn: update SVRG}) and (\ref{eqn: update RBCD}), $m$ updating steps are taken in the follow-up inner loops.
Different from MRBCD, two additional coupling steps are added to ensure the acceleration of our method: line \ref{eqn: coupling I} uses the two momentum technique proposed by \cite{allen2016katyusha}, and line \ref{eqn: coupling II} is a common practice in accelerated methods such as \cite{nesterov2012efficiency,fercoq2015accelerated}.
Note that we write line \ref{eqn: VR gradient} just for ease of notation and only $[\vB_k]_l$ is needed in practice.
We will show that AVRBCD enjoys the accelerated convergence rate in the next section.
\subsection{Implementation without Full Vector Operation}
Similar to the methods in \cite{nesterov2012efficiency,lee2013efficient,fercoq2015accelerated}, AVRBCD I requires $\OM(d)$ computation in each inner loop due to the convex combination in line \ref{eqn: coupling I}, which invalidates the advantage of low per-iteration complexity in BCD type methods.
Borrowing ideas from \cite{lee2013efficient,fercoq2015accelerated}, we propose a more practical implementation of AVRBCD in Algorithm \ref{alg: AVRBCD II}, avoiding the full vector operation in the inner loop.
Three functions $\{\bar{\xB}_k, \bar{\yB}_k, \bar{\zB}_k\}$ are used in AVRBCD II and we give their definitions here
\begin{align}
	\bar{\yB}_k =&~ \beta_{j-1}\uB_{j-1}^s + \gamma_s\hat{\zB}_{j-1}^s  + \dot{\xB}^s,\label{eqn: alg equiv I}\\
	\bar{\zB}_k =&~ \hat{\zB}_j^s + \dot{\xB}^s, \label{eqn: alg equiv II}\\
	\bar{\xB}_k =&~ \beta_{j-1}\uB_j^s + \gamma_s\hat{\zB}_j^s  + \dot{\xB}^s, \label{eqn: alg equiv III}
\end{align}
with $k = sm+j$.
In order to see the equivalence between Algorithm \ref{alg: AVRBCD I} and \ref{alg: AVRBCD II}, we give the following proposition.
\begin{proposition}
		$\bar{\xB}_k = \xB_k$, $\bar{\yB}_k = \yB_k$, and $\bar{\zB}_k = \zB_k$ hold for all $k$, if $\bar{i}_k = i_k$, $\bar{l}_k = l_k$, and $\sigma_s = \bar{\sigma}_s$ for all $k$ and $s$, where $\xB_k$, $\yB_k$, $\zB_k$, $i_k$, and $l_k$ are in Algorithm \ref{alg: AVRBCD I}, and $\bar{\xB}_k$, $\bar{\yB}_k$ $\bar{\zB}_k$, $\bar{i}_k$, $\bar{l}_k$, and $\bar{\sigma}_s$ are defined in (\ref{eqn: alg equiv I}), (\ref{eqn: alg equiv II}), (\ref{eqn: alg equiv III}), and Algorithm \ref{alg: AVRBCD II}.
\end{proposition}
\begin{proof}
	First, we prove that if at the beginning of the $s^{th}$ epoch, i.e. $j = 0$ and $k = sm$, $\bar{\zB}_k = \zB_k$ and $\bar{\xB}_k = \xB_k$ stand, then all these three equations stand in the following iterations in that epoch.
	We prove with induction.
	Assume that the equivalence holds up till the ${\kappa-1}^{th}$ iteration.
	Since $\sigma_s = \bar{\sigma}_s$, we have $\tilde{\xB}^s = \dot{\xB}^s$.
	In the $\kappa^{th}$ iteration, for $\bar{\yB}_\kappa$ we have
	\begin{align*}
	\yB_\kappa &= \alpha_{1,s}\bar{\xB}_{\kappa-1} +\alpha_{2,s}\bar{\zB}_{\kappa-1} + \alpha_{3,s}\dot{\xB}^s\\
	&= \alpha_{1,s}(\beta_{j-2}^s\uB_{j-1}^s + \gamma_s\hat{\zB}_{j-1}^s) + \alpha_{2,s}\hat{\zB}_{j-1}^s + \dot{\xB}^s \\
	&= \beta_{j-1}^s\uB_{j-1}^s + \gamma_s\hat{\zB}_{j-1}^s + \dot{\xB}^s = \bar{\yB}_\kappa.
	\end{align*}
	since $\alpha_{1,s}\beta_{j-2}^s = \beta_{j-1}^s$ and $\alpha_{1,s}\gamma_s + \alpha_{2,s} = \gamma_s$.
	Suppose the $l^{th}$ block is selected in that section, by induction we have 
	\begin{equation*}
		[\bar{\zB}_\kappa]_{\backslash l} = [\zB_{\kappa -1}]_{\backslash l} = [\zB_{\kappa}]_{\backslash l}
	\end{equation*}
	and 
	\begin{align*}
		[\bar{\zB}_\kappa]_l = [\zB_j^s + \dot{\xB}^s]_l &= \prox_{\eta\PB_l}([\hat{\zB}^s_{j-1} + \dot{\xB}^s - \eta\tilde{\nabla}_\kappa]_l) \\
		& = \prox_{\eta\PB_l}([\bar{\zB}_{\kappa -1} - \eta\tilde{\nabla}_\kappa]_l) \\
		& = \prox_{\eta\PB_l}([\zB_{\kappa -1} - \eta\tilde{\nabla}_\kappa]_l)  = [\zB_\kappa]_l
	\end{align*}
	Thus we have $\bar{\zB}_{\kappa} = \zB_{\kappa}$.
	For $\bar{\xB}_\kappa$, we have
	\begin{align*}
		\xB_{\kappa} &= \bar{\yB}_{\kappa} + \alpha_{2,s}B(\hat{\zB}_{j}^s - \hat{\zB}_{j-1}^s) \\
		&=\beta_{j-1}^s\uB_{j-1}^s+\gamma_s\hat{\zB}_{j-1}^s + \alpha_{2,s}B(\hat{\zB}_{j}^s - \hat{\zB}_{j-1}^s) + \dot{\xB}^s\\
		&= \beta_{j-1}^s\uB_{j}^s + \gamma_s \hat{\zB}_{j}^s + \dot{\xB}^s = \bar{\xB}_{\kappa}
	\end{align*}
	due to the updating rule of $\uB_k$ in line \ref{eqn: lazy update} in AVRBCD II.
	
	We then show that at the beginning of each epoch $\bar{\zB}_k = \zB_k$ and $\bar{\xB}_k = \xB_k$ stand.
	The idea here is that we are using iterates $\hat{\zB}_j^s$ and $\uB_j^s$ from successive epochs to represent the same $\xB_k$ and $\zB_k$.
	For $\zB_k$, when $k = 0$, 
	\begin{equation*}
		\bar{\zB}_0 = \hat{\zB}_0^0 + \dot{\xB}^0 = \xB_0 = \zB_0.
	\end{equation*}
	When $k = (s+1)m$ with $s \geq 0$, we have 
	\begin{equation*}
		\zB_{(s+1)m} \stackrel{a}{=} \bar{\zB}_{sm+m}  \stackrel{b}{=} \hat{\zB}_0^{s+1} + \dot{\xB}^{s+1} = \bar{\zB}_{(s+1)m}
	\end{equation*}
	where the equation $a$ is from the induction in previous epoch, and equation $b$ is from the definition of $\hat{\zB}_0^{s}$ in line \ref{eqn: update between epochs} in AVRBCD II.
	For $\xB_k$, we set $\beta_{-1}^s = 1$ (which is not used in practice).
	When $k = 0$, we clearly have $\xB_0 = \bar{\xB}_0$.
	When $k = (s+1)m$ with $s \geq 0$, we have 
	\begin{equation*}
		\xB_{(s+1)m} \stackrel{a}{=} \bar{\xB}_{sm+m} \stackrel{b}{=} \beta_{-1}^{s}\uB_0^{s+1} + \gamma_{s}\hat{\zB}_0^{s} + \dot{\xB}^{s} = \bar{\xB}_{(s+1)m}.
	\end{equation*}
	where the equation $a$ is from the induction in previous epoch, and equation $b$ is from the definition of $\uB_0^{s}$ from line \ref{eqn: update between epochs II} in AVRBCD II.
	Thus we have the result.
\end{proof}
\subsection{Numerical Stability}
Careful readers might have noticed that, since $\beta_j^s$ decreases exponentially (line \ref{eqn: numerical issue} in Algorithm \ref{alg: AVRBCD II}), the computation of $\uB_j^s$ involving the inversion of $\beta_j^s$ can be numerically unstable.
To overcome this issue, we can simply keep their product $\beta^s_j\uB_j^s = \xi\in \RBB^d$ in stead of themselves separately to make the computation numerically tractable.
We also update $\xi$ in a "lazy" manner so as to avoid high computation complexity.
At the beginning of each epoch, we initialize a count vector $\omega \in \RBB^B$ to be a zero vector.
In each of the following iterations, we do the follow steps
\begin{enumerate}
	\item set $\omega_i = \omega_i + 1$ for all $i$;
	\item leave $[\xi]_{\backslash l}$ unchanged but update only $[\xi]_l$ as
	\begin{equation}
	[\xi]_l = \alpha_{1,s}^{\omega_l}[\xi]_l - (\alpha_{2,s}B+\gamma_s)(\hat{\zB}_j^s - \hat{\zB}_{j-1}^s)
	\end{equation}
	where $l$ is the block selected in that iteration;
	\item set $\omega_l = 0$.
\end{enumerate}
The idea is to use $\omega$ to record the exponential of $\alpha_{1,s}$ needs to be multiplied later.
In this way, the exact computation of $\beta^s_j\uB_j^s$ only happens at the end of each epoch.
\subsection{Computational Complexity}
Since the complexity of computing the partial gradient of some general convex function $f_i(\cdot)$ can be difficult to analyze, we mainly focus on the well-know Empirical Risk Minimization (ERM) problem with linear predictor, same as that in \cite{fercoq2015accelerated}.
Under this setting, we can analyze the computational complexity of each inner loop as follows.
\begin{enumerate}
	\item $\OM(\Omega + B)$ from line \ref{eqn: lazy update}, where $\OM(B)$ is from the lazy update.
	\item $\OM(\rho d + B)$ from line \ref{eqn: partial gradient}, the computation of partial gradient, since (i) $\aB_i^\top\dot{\xB}^s$ can be kept when computing $\mu^s$, (ii) $\aB_i^\top\hat{\zB}_j^s$ can be computed in $\OM(\rho d)$, and (iii) $\aB_i^\top \beta_j^s\uB_j^s$ can be computed as $\sum_{l=1}^{B} \alpha_{1,s}^{\omega_l}[\aB_i]_l^\top[\xi]_l$ in $\OM(\rho d + B)$ where $\xi$ is defined above.
	Recall that $\aB_i^\top\bar{\yB}_k = \aB_i^\top(\beta_{j-1}\uB_{j-1}^s + \gamma_s\hat{\zB}_{j-1}^s  + \dot{\xB}^s) $.
	\item $\OM(\Omega)$ from the others.
\end{enumerate}
Thus the overall computation complexity is $\OM(\rho d + B + \Omega)$.
When we pick a moderate $B$ and the sparsity $\rho$ is small, $\Omega$ will dominate the other two terms.
This is the same per-iteration complexity as MRBCD and is much smaller than $\OM(d)$ in methods like Katyusha and SVRG.

\begin{algorithm}[t]
	\caption{AVRBCD II}
	\begin{algorithmic}[1]
		\label{alg: AVRBCD II}
		\REQUIRE $m, \xB_0, \alpha_{1,0}, \alpha_{2, 0}$
		\STATE $\uB_0^0 = \hat{\zB}_0^0 \leftarrow 0; \dot{\xB}^0 \leftarrow \xB_0$;
		\FOR{$s \leftarrow 0$ \TO $S$}
		\STATE $\alpha_{2,s} = \frac{\sqrt{\alpha_{2, s-1}^4 + 4\alpha_{2, s-1}^2} - \alpha_{2, s-1}^2}{2}$;
		\STATE $\alpha_{1, s} = \alpha_{1, s-1}(1-\alpha_{2, s}), \alpha_{3, s} = 1 - \alpha_{1, s} - \alpha_{2, s}$; 
		\STATE $\bar{L}_s = \frac{L_Q}{B\alpha_{3, s}} + L_B, \eta_s = \frac{1}{\bar{L}_s\alpha_{2,s}B}$;
		\STATE $\mu^s = \nabla f(\dot{\xB}^s)$; \label{eqn: full gradient}
		\FOR{$j \leftarrow 1$ \TO $m$}
		\STATE $k = (sm)+j$;
		\STATE sample $i$ from $\{1, \ldots, n\}$ and $l$ from $\{1, \ldots, B\}$; \label{eqn: sample}
		\STATE $\tilde{\nabla}_k = \mu^s + \nabla f_i(\bar{\yB}_k) - \nabla f_i(\bar{\yB}_k)$; \label{eqn: partial gradient}
		\STATE $[\hat{\zB}_j^s]_l = \prox_{\eta\PB_l}([\bar{\zB}_k - \eta\tilde{\nabla}_k]_l) - [\dot{\xB}^s]_l$;
		\STATE $[\hat{\zB}_j^s]_{\backslash l} = [\hat{\zB}_{j-1}^s]_{\backslash l}$;
		\STATE $\uB_j^s = \uB_{j-1}^s + \frac{\alpha_{2,s}B - \gamma_s}{\beta_{j-1}^s}(\hat{\zB}_j^s - \hat{\zB}_{j-1}^s)$; \label{eqn: lazy update}
		\STATE $\beta_j^s = \alpha_{1,s}\beta_{j-1}^s$ \label{eqn: numerical issue};
		\ENDFOR
		\STATE Sample $\bar{\sigma}_s$ from $\{1, \ldots, m\}$ uniformly;
		\STATE $\dot{\xB}^{s+1} = \beta_{\bar{\sigma}_s-1}^s\uB_{\bar{\sigma}_s}^s+\gamma_s\hat{\zB}_{\bar{\sigma}_s}^s + \dot{\xB}^s$;
		\STATE $\xB = \bar{\xB}_k - \dot{\xB}^{s+1}$, $\hat{\zB}_0^{s+1} = \bar{\zB}_k - \dot{\xB}^{s+1}$; \label{eqn: update between epochs}
		\STATE $\beta_0^{s+1} = \alpha_{1,s+1}, \uB_0^{s+1} = \xB - \gamma_{s+1} \hat{\zB}_0^{s+1}$	\label{eqn: update between epochs II}
		\ENDFOR
	\end{algorithmic}
\end{algorithm}
\subsection{AVRBCD with Active Set}
In MRBCD III, the authors use an active set strategy to further accelerate their method when solving sparse learning problems.
We adapt such strategy to AVRBCD by modifying only two lines in AVRBCD II.
\begin{enumerate}
	\item Add an operation 
	\begin{equation} \label{eqn: snapshot proximal}
		\dot{\xB}^s = \prox_{\frac{1}{L}\PB}(\dot{\xB}^s - \frac{1}{L}\mu^s)
	\end{equation}
	 below line \ref{eqn: full gradient}.
	 \item In line \ref{eqn: sample}, after we have selected block $l$, skip the rest operations in this iteration if $[\dot{\xB}^s]_l = 0$.
\end{enumerate}
In the first modification, the idea is to fully utilize the full gradient $\mu^s$ and produce a sparser snapshot $\dot{\xB}^s$ with a proximal step.
An empirical observation in our experiments suggests that the support of the sparser snapshot provides a good prediction of the support of the optimal point $\xB^*$, thus we omit the update on blocks out of the support of $\dot{\xB}^s$ in the second modification.
Such active set strategy is common in the RBCD literature, and usually boosts the empirical performance \cite{friedman2007pathwise,wu2008coordinate}.
\section{Convergence Analysis} \label{section: Convergence Analysis}
	We give the convergence results of AVRBCD under both non-proximal ($\PB(\xB) \equiv 0$) and proximal ($\PB(\xB) \neq 0$) settings.
	In the former case, we show that AVRBCD takes $\OM({(n+\sqrt{nL})}/{\sqrt{\epsilon}})$ iterations to obtain an $\epsilon$-accurate solution, while in the latter case, $\OM({\sqrt{B}(n+\sqrt{nL})}/{\sqrt{\epsilon}})$ iterations are needed to achieve the same accuracy.
	We believe the additional $\sqrt{B}$ factor is the artifact of our proof as we do not observe such phenomenon in experiments.
	\subsection{Non-Proximal Case ($\PB(\xB) \equiv 0$)}
	First, let us establish an inequality that relates the objective values between two successive iterations.
	Define 
	\begin{equation*}
		{d}(\xB) = \FB(\xB) - \FB(\xB^*)
	\end{equation*}
	to be the sub-optimality at $\xB$, we have the following lemma.
	\begin{lemma} \label{lemma: non-proximal}
		In Algorithm \ref{alg: AVRBCD I}, we have
		\begin{align*}
			\EBB_{l, i_k} &d(\xB_k) \leq \alpha_{3,s}d(\tilde{\xB}^s) + \alpha_{1,s}d(\xB_{k-1})\\
			&+ \frac{\bar{L}_s\alpha_{2,s}^2B^2}{2}(\|\xB^* - \zB_{k-1}\|^2 - \EBB_{l, i_k}\|\xB^* - \zB_k\|^2)\\
		\end{align*}
	\end{lemma}
	The proof of this lemma in the appendix. We will prove Theorem \ref{theorem: non-proximal} based on Lemma \ref{lemma: non-proximal}.
	\begin{theorem}[Non-Proximal] \label{theorem: non-proximal}
		By setting $\alpha_{2, 0} = \frac{2}{\nu}$, $0<\alpha_{3, 0}\leq \frac{\nu - 2}{\nu}$, and $\alpha_{1, 0} = 1 - \alpha_{2, 0} - \alpha_{3, 0}$, with $\nu > 2$, we have
		\begin{equation*}
			d(\tilde{\xB}^s) \leq \frac{\alpha_{2,s}^2}{\alpha_{3,s}}(\frac{\alpha_{1,0}}{\alpha_{2,0}^2}\frac{d(\xB_0)}{m} + \frac{\alpha_{3,0}}{\alpha_{2,0}^2} d(\xB_0) + \frac{\bar{L}_0B^2}{2m}\|\xB^* - \xB_0\|^2)
		\end{equation*}
	\end{theorem}
	\begin{proof}
		The expectations are taken with respect to all history randomness, and are omitted for simplicity.
		Use $d_k$ to denote $d(\xB_k)$ and $\tilde{d}_s$ to denote $d(\tilde{\xB}^s)$.
		Dividing both sides of Lemma \ref{lemma: non-proximal} by $\alpha_{2,s}^2$ and summing from $k = sm + 1$ to $(s+1)m$, we get
		\begin{align*}
			\frac{1}{\alpha_{2,s}^2} &\sum_{j = 1}^{m}d_{sm+j} \leq \frac{\alpha_{1,s}}{\alpha_{2,s}^2} \sum_{j = 0}^{m-1} d_{sm+j} + \frac{\alpha_{3,s}}{\alpha_{2,s}^2}m\tilde{d}^s \\
			&+\frac{\bar{L}_sB^2}{2}\{\|\xB^* - \zB_{sm}\|^2 -\|\xB^* - \zB_{sm+m}\|^2\}.
		\end{align*}
		By rearranging terms, we have
		\begin{align*}
			\frac{\alpha_{1,s}}{\alpha_{2,s}^2}&d_{sm+m} + \frac{1-\alpha_{1,s}}{\alpha_{2,s}^2} \sum_{j = 1}^{m}d_{sm+j} \leq \frac{\alpha_{1,s}}{\alpha_{2,s}^2}d_{sm}+ \frac{\alpha_{3,s}}{\alpha_{2,s}^2}m\tilde{d}_s\\
			 &+ \frac{\bar{L}_sB^2}{2}\{\|\xB^* - \zB_{0, s}\|^2 -\|\xB^* - \zB_{m, s}\|^2\}.
		\end{align*}
		Using the fact that 
		$\frac{1 - \alpha_{1,s}}{\alpha_{2,s}^2} = \frac{\alpha_{3,s+1}}{\alpha_{2,s+1}^2}, 
		\frac{1}{\alpha_{2,s}^2} = \frac{1 - \alpha_{2,s+1}}{\alpha_{2,s+1}^2},$
		we have $\frac{\alpha_{1,s}}{\alpha_{2,s}^2} = \frac{\alpha_{1,s+1}}{\alpha_{2,s+1}^2}$.
		From the definition of $\tilde{\xB}^s$, we have $m\tilde{d}_{s+1}\leq\sum_{j=1}^{m}d_{sm+j}$.
		Additionally, using $\bar{L}_{s+1}\leq\bar{L}_s$, we have the following inequality
		\begin{align*}
			\frac{\alpha_{1,s+1}}{\alpha_{2,s+1}^2}&d_{sm+m} + \frac{\alpha_{3,s+1}}{\alpha_{2,s+1}^2} m\tilde{d}_{s+1} + \frac{\bar{L}_{s+1}B^2}{2}\|\xB^* - \zB_{sm+m}\|^2 \\
			&\leq  \frac{\alpha_{1,s}}{\alpha_{2,s}^2}d_{sm} + \frac{\alpha_{3,s}}{\alpha_{2,s}^2}m \tilde{d}_s + \frac{\bar{L}_sB^2}{2}\|\xB^* - \zB_{sm}\|^2.
		\end{align*}
		By the non-negativity of $d(\cdot)$, we have the result.
	\end{proof}
	\noindent The following relations come from the constructions of $\{\alpha_{i,s}\}_{i=1}^3$,
	\begin{enumerate}
		\item $\alpha_{2, s} \leq \alpha_{2, s-1}$, $\alpha_{1, s} \leq \alpha_{1, s-1}$, and thus $\alpha_{3, s} \geq \alpha_{3, s-1}$;
		\item $\alpha_{2, s} \leq 2/(s+\nu)$, if $\alpha_{2, s-1} \leq 2/(s+\nu - 1)$.
	\end{enumerate}
	From such relations, we have the corollary to describe the convergence rate of AVRBCD in non-proximal case.
	\begin{corollary}
		By setting $m = Bn$, we have
		\begin{equation*}
			d(\tilde{\xB}^s) \leq \frac{C d(\xB_0) + \frac{\tilde{L}}{n}\|\xB^* -\xB_0\|^2}{s^2}
		\end{equation*}
		where $C$ is some constant.
		
		In other words, to obtain an $\epsilon$-accurate solution, AVRBCD need $\sqrt{{(C d(\xB_0) + \frac{\tilde{L}}{n}\|\xB^* -\xB_0\|^2)}/{\epsilon}}$ iterations.
	\end{corollary}
	\subsection{Proximal Case ($\PB(\xB) \neq 0$)}
	The key idea to prove the convergence of proximal version of AVRBCD is to express $\xB_k$ as the convex combination of $\{\tilde{\xB}^i\}_{i=0}^s$ and $\{\zB_l\}_{l=0}^k$.
	\begin{lemma}
		\label{lemma: convex combination}
		In Algorithm \ref{alg: AVRBCD I}, by setting $\alpha_{2, 0} = \alpha_{3, 0} =1/2B$, for $k = sm+j \geq 1$, we have
		\begin{equation}
			\xB_k = \sum_{i=0}^{s-1}\lambda_k^i \tilde{\xB}^i + \beta_j^s \tilde{\xB}^s + \sum_{l=0}^{k} \gamma_k^l \zB_l
		\end{equation}
		where $\gamma_0^0 = 1$, $\gamma_1^0 = \frac{1}{2} - \frac{1}{2B}$, $\gamma_1^1 = \frac{1}{2}$, $\beta^0_0 = 0$, $\lambda_{(s+1)m}^s = \beta_m^s$, $\lambda^i_{k+1} = \alpha_{1,s}\lambda^i_k $
		\begin{equation*}
		\gamma_{k+1}^l = 
		\begin{cases}
		\alpha_{1,s}\gamma_k^l,~&l=0,\ldots,k-1\\
		B\alpha_{1,s}\alpha_{2,s}+(1-B)\alpha_{2,s},~&l=k \\
		B\alpha_{2,s},~&l=k+1
		\end{cases}
		\end{equation*}
		and
		\begin{equation}
		\beta_{j+1}^s = \alpha_{1,s}\beta_j^s+\alpha_{3,s}.
		\end{equation}
		Additionally, we have $\sum_{i=0}^{s-1}\lambda_k^i + \beta_j^s + \sum_{l=0}^k \gamma_k^l = 1$ and each entry in this sum is non-negative for all $k\geq1$, i.e. $\xB_k$ is a convex combination of $\{\tilde{\xB}^i\}_{i=0}^s$ and $\{\zB_l\}_{l=0}^k$.
	\end{lemma}
	
	From Lemma \ref{lemma: convex combination} and the convexity of $\PB(\cdot)$, we have 
	\begin{equation*}
		\PB(\xB_k) \leq \sum_{i=0}^{s-1}\lambda_k^i \PB(\tilde{\xB}^i) + \beta_j^s \PB(\tilde{\xB}^s) + \sum_{l=0}^{k} \gamma_k^l \PB(\zB_l) \defi \hat{\PB}(\xB_k).
	\end{equation*}
	We also define the sub-optimality $d(\xB)$ and its upper bound $\hat{d}(\xB_k)$ at $\xB_k$ as
	\begin{align*}
		d(\xB_k) = (\FB(\xB_k) + \PB(\xB_k)) - (\FB(\xB^*) + \PB(\xB^*)), \\
		\hat{d}_k = (\FB(\xB_k) + \hat{\PB}(\xB_k)) - (\FB(\xB^*) + \PB(\xB^*)).
	\end{align*}
	For $\hat{d}_k$, we have $0\leq d(\xB_k) \leq \hat{d}_k$ and $d(\xB_0) = \hat{d}_0$.
	\begin{lemma}
		\label{lemma: key}
		In Algorithm \ref{alg: AVRBCD I}, by setting $\alpha_{2, 0} = \alpha_{3, 0} =1/2B$,
		\begin{align*}
			\EBB_{l, i_k} &\hat{d}_k \leq \alpha_{3,s}d(\tilde{\xB}^s) + \alpha_{1,s}\hat{d}_{k-1}\\
			&+ \frac{\bar{L}_s\alpha_{2,s}^2B^2}{2}(\|\xB^* - \zB_{k-1}\|^2 - \EBB_{l, i_k}\|\xB^* - \zB_k\|^2)\\
		\end{align*}
	\end{lemma}
	This lemma is similar to Lemma \ref{lemma: non-proximal}, but harder to prove due to the regularization term $\PB$.
	Again, we use it to prove Theorem \ref{theorem: proximal}. The proof is similar to that of Theorem \ref{theorem: non-proximal}.
	\begin{theorem} \label{theorem: proximal}
		By setting $\alpha_{2, 0} = \alpha_{3, 0} = \frac{1}{2B}$ and $\alpha_{1, 0} = 1 - \alpha_{2, 0} - \alpha_{3, 0}$, we have
		\begin{equation*}
			d(\tilde{\xB}^s) \leq \frac{\alpha_{2,s}^2}{\alpha_{3,s}}(\frac{\alpha_{1,0}}{\alpha_{2,0}^2}\frac{d(\xB_0)}{m} + \frac{\alpha_{3,0}}{\alpha_{2,0}^2} d(\xB_0) + \frac{\bar{L}_0B^2}{2m}\|\xB^* - \xB_0\|^2)
		\end{equation*}
	\end{theorem}
	Again, we have the corollary to describe the convergence rate of  AVRBCD in proximal case.
	\begin{corollary}
		By setting $m = Bn$, we have
		\begin{equation*}
		d(\tilde{\xB}^s) \leq \frac{B(C d(\xB_0) + \frac{\tilde{L}}{n}\|\xB^* -\xB_0\|^2)}{s^2}
		\end{equation*}
		where $C$ is some constant.
		
		In other words, to obtain an $\epsilon$-accurate solution, AVRBCD need $\sqrt{{B(C d(\xB_0) + \frac{\tilde{L}}{n}\|\xB^* -\xB_0\|^2)}/{\epsilon}}$ iterations.
	\end{corollary}

	\subsection{Overall Complexity}
	Combining with the analysis of per-iteration complexity in the previous section and the convergence rate discussed above, the overall computational complexity of AVRBCD in the sparse ERM problems is $\OM({nd}/{\sqrt{\epsilon}}+{d\sqrt{\tilde{L}n}}/{\sqrt{\epsilon}})$ for non-proximal case, and $\OM({nd\sqrt{B}}/{\sqrt{\epsilon}}+{d\sqrt{\tilde{L}nB}}/{\sqrt{\epsilon}})$
	for proximal case.	
	This is the similar to Katyusha, i.e. $\OM({nd}/{\sqrt{\epsilon}}+{d\sqrt{L_{max}n}}/{\sqrt{\epsilon}})$ in both proximal or non-proximal case, and is better than $\OM({ndBL_B}/{\sqrt{\epsilon}})$ in APPROX and $\OM((nd+dL_{max}/\epsilon)\log\frac{1}{\epsilon})$ in MRBCD II.
\section{Experiments}
	\begin{table}[t]
		\centering
		\caption{Statistics of datasets.}
		\small
		\begin{tabular}{|c|c|c|c|}
			\hline
			Dataset      &     n     &    d     & sparsity \\ \hline
			real-sim        & $72,309$  &  $20,958$   &    $0.24\%$    \\ \hline
			rcv1       & $20,242$  &   $47,236$   &    $0.16\%$    \\ \hline
			news20.binary        & $19,996$  &  $1,355,191$   &    $0.0336\%$    \\ \hline
		\end{tabular}
		\label{table: statistics}
	\end{table}
	In this section, we present results of several numerical experiments to validate our analysis for AVRBCD and to show the effectiveness of AVRBCD with active set (AVRBCD-AC) on real problems.
	Empirical studies on $l_1$-logistic regression and $l_1l_2$-logistic regression are conducted.
	Three large scale datasets from LibSVM are used, namely real-sim, rcv1, and new20.binary, all of which have large number of samples and features ($n, d > 10^4$).
	The statistics are given in Table \ref{table: statistics} along with the sparsity of the datasets.	
	Katyusha with opt.II~\cite{allen2016katyusha}, MRBCD II and III \cite{zhao2014accelerated}, and SVRG \cite{johnson2013accelerating} are included in comparison.
	We use the parameter suggested in the original paper for Katyusha.
	For SVRG, we set the inner loop count $m = n$ and the step size $\eta = {1}/{2L_{max}}$.
	We also incorporate the mini-batch technique in our methods and set the mini batch size to $8$, same as that of MRBCD II and III.
	The step sizes for MRBCD II and III are set to ${4}/{L_{max}}$ which gives the best performance in our experiments.
	For AVRBCD and AVRBCD-AC, the step sizes are set to ${4}/{L_{max}\alpha_{2,s}}$ which increases are iteration goes on, similar to Katyusha.
	For MRBCD II and III, AVRBCD, and AVRBCD-AC, we set $m = nB/8$.
	As for initialization, $\xB_0$ is set to zero in all experiments.
	We define the \emph{log-suboptimality} at $\xB$ as $\log_{10} {(\FB^\PB(\xB)-\FB^\PB(\xB^*))}$ and the \emph{effective pass} as the evaluation of $nd$ component partial gradients.
	These quantities are used to evaluate the performance of the algorithms \cite{zhao2014accelerated}.
	Due to the randomness of the algorithms, the reported results are the average of 10 independent trials.
	\begin{figure*}[t]
		\centering
		\begin{tabular}{c@{}c@{}c}
			\includegraphics[width = .33\columnwidth]{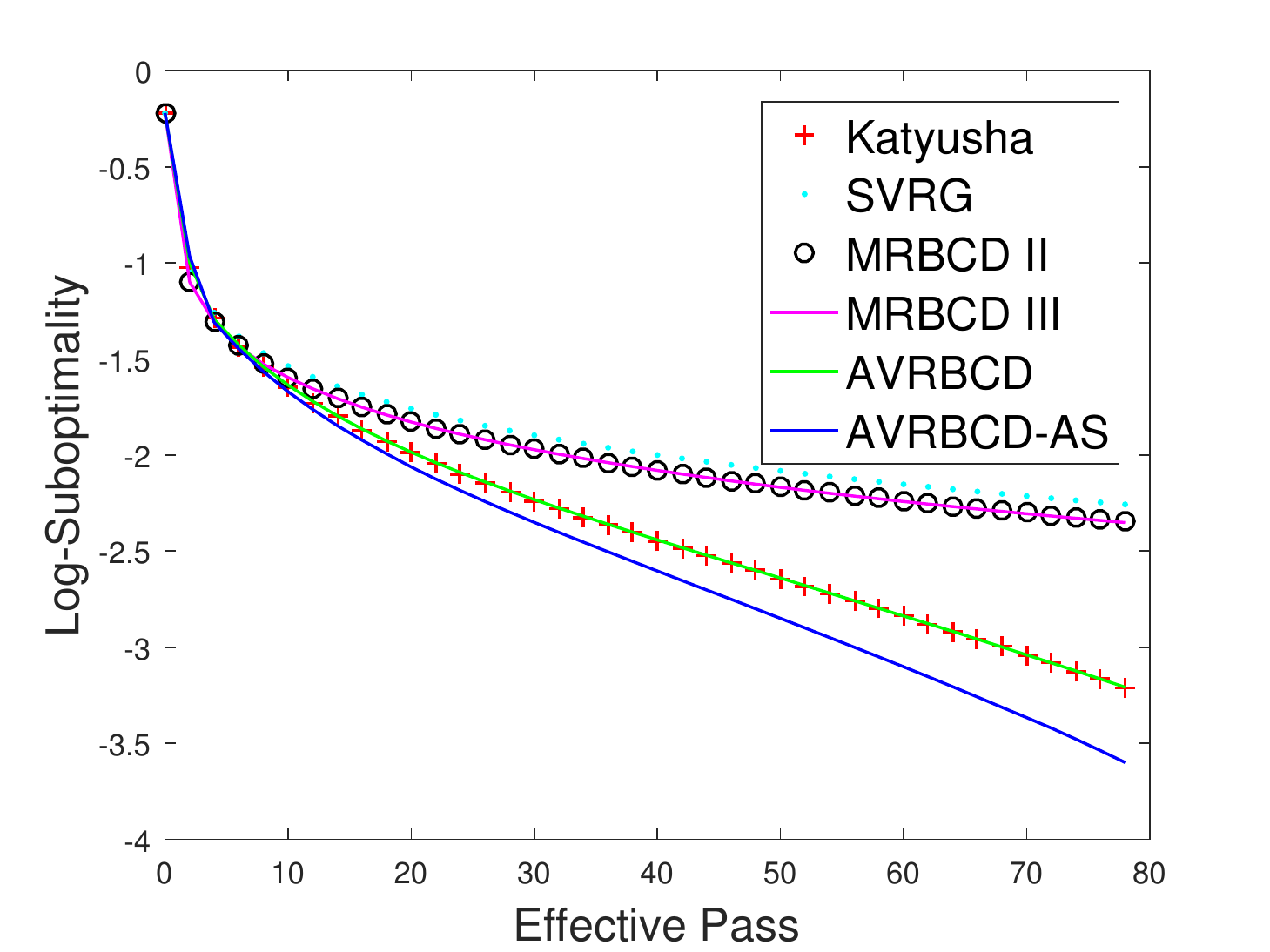} 
			\includegraphics[width = .33\columnwidth]{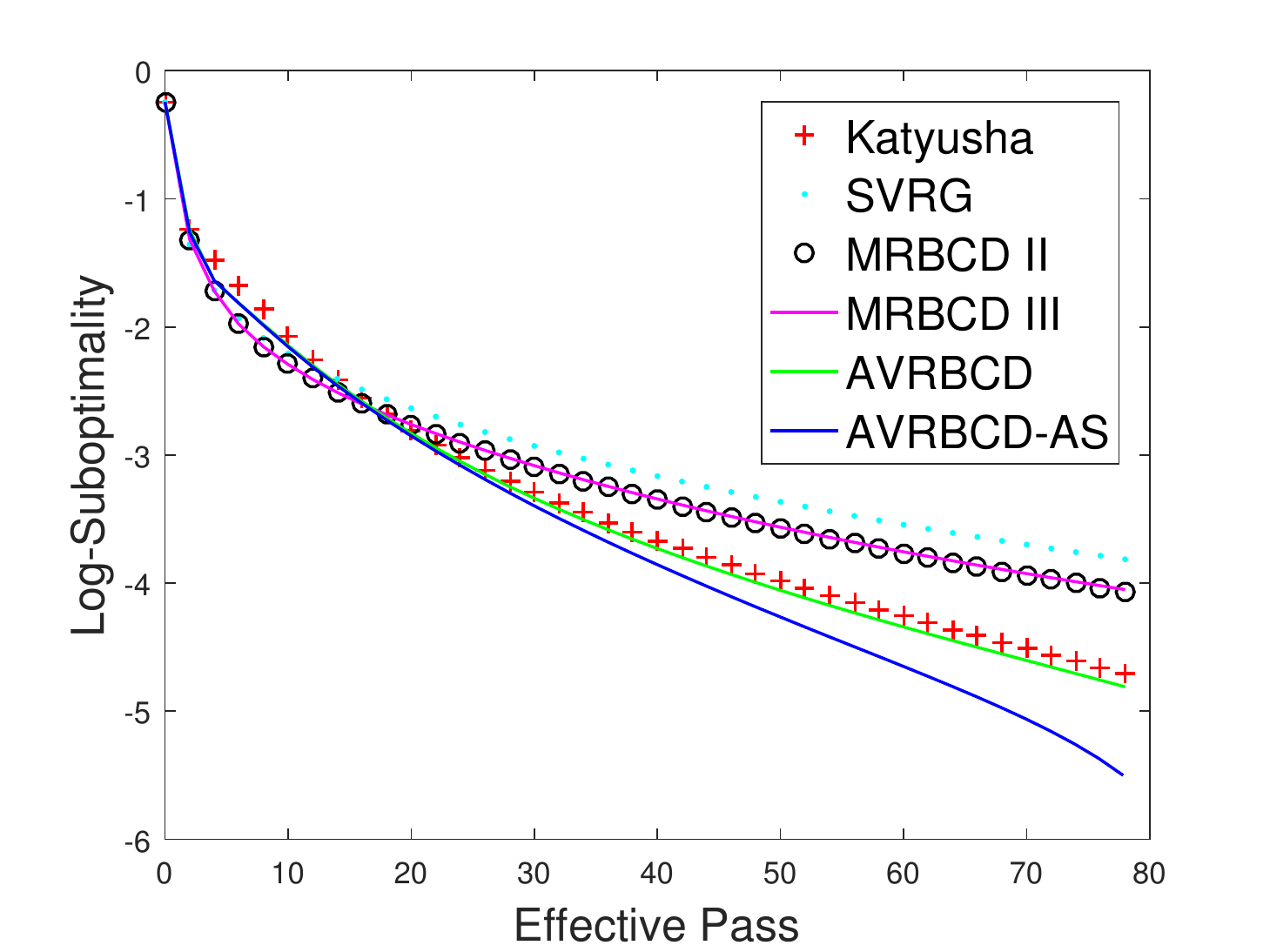} 
			\includegraphics[width = .33\columnwidth]{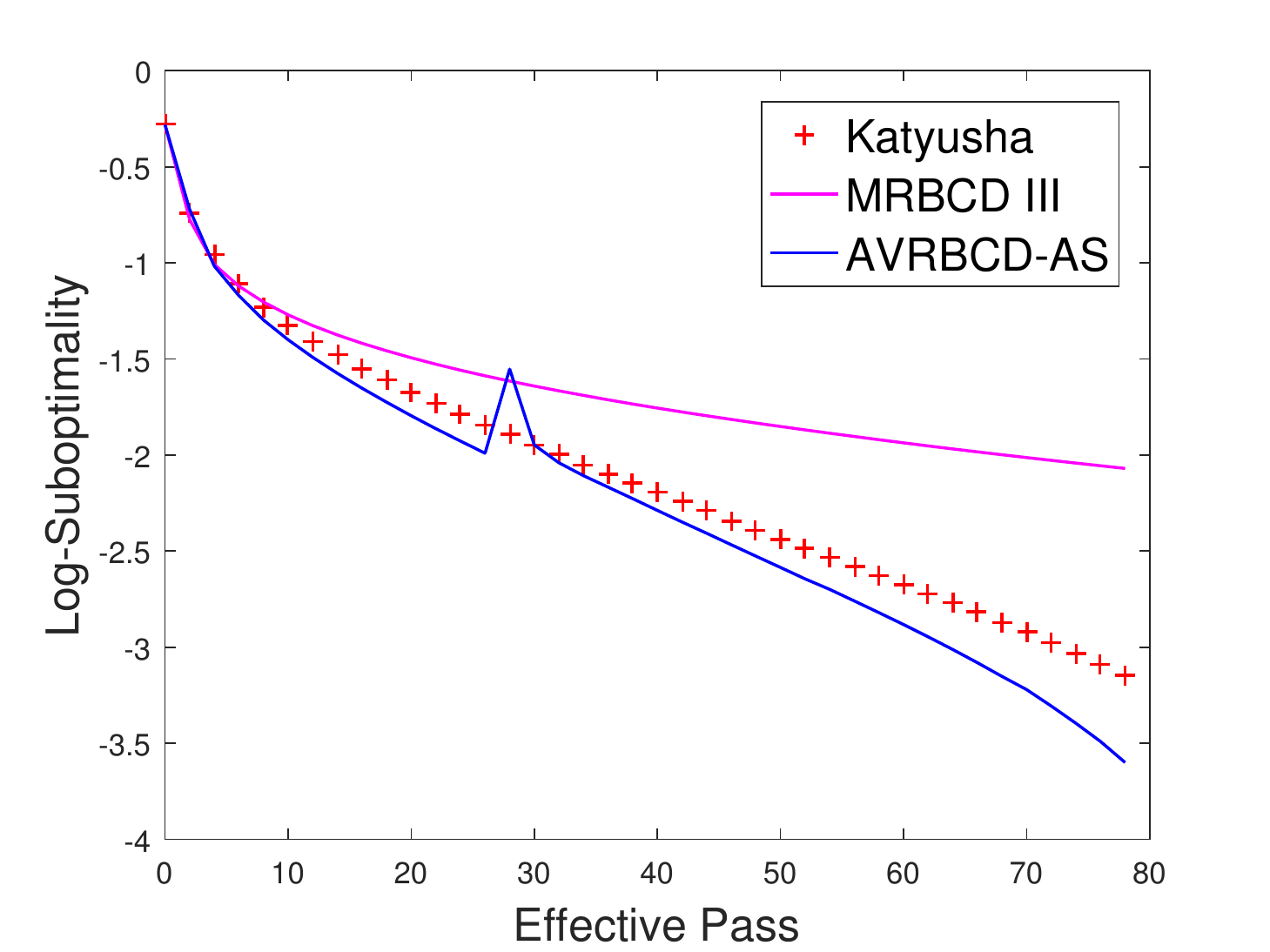}
		\end{tabular}
		\caption{$l_1$-Logistic Regression. From left to right are results on rcv1, real\_sim, and news20.binary}
		\label{fig: l1 LR}
	\end{figure*}
	\begin{figure}[t]
		\begin{tabular}{c@{}c}
			\includegraphics[width = .5\columnwidth]{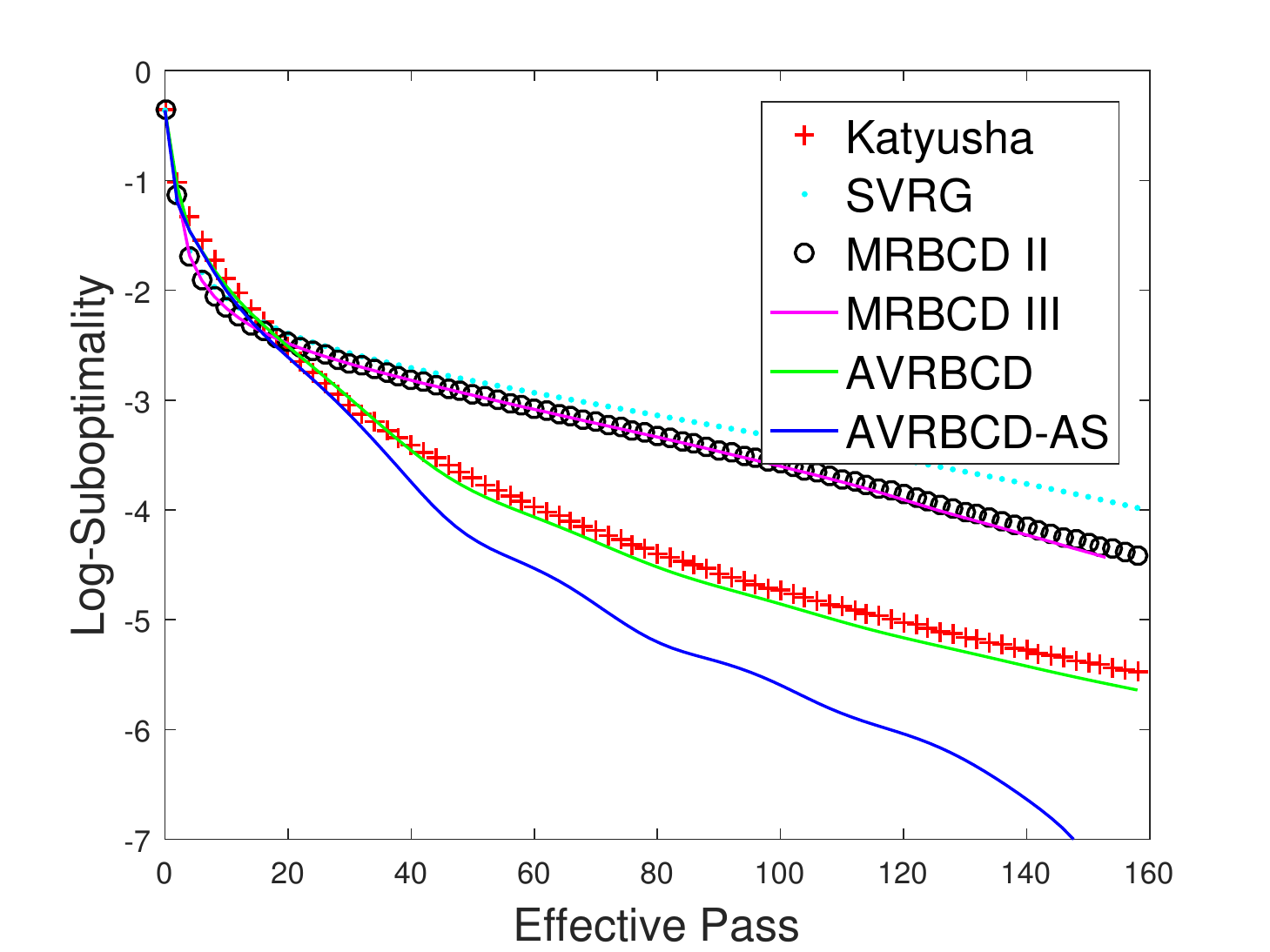} 
			\includegraphics[width = .5\columnwidth]{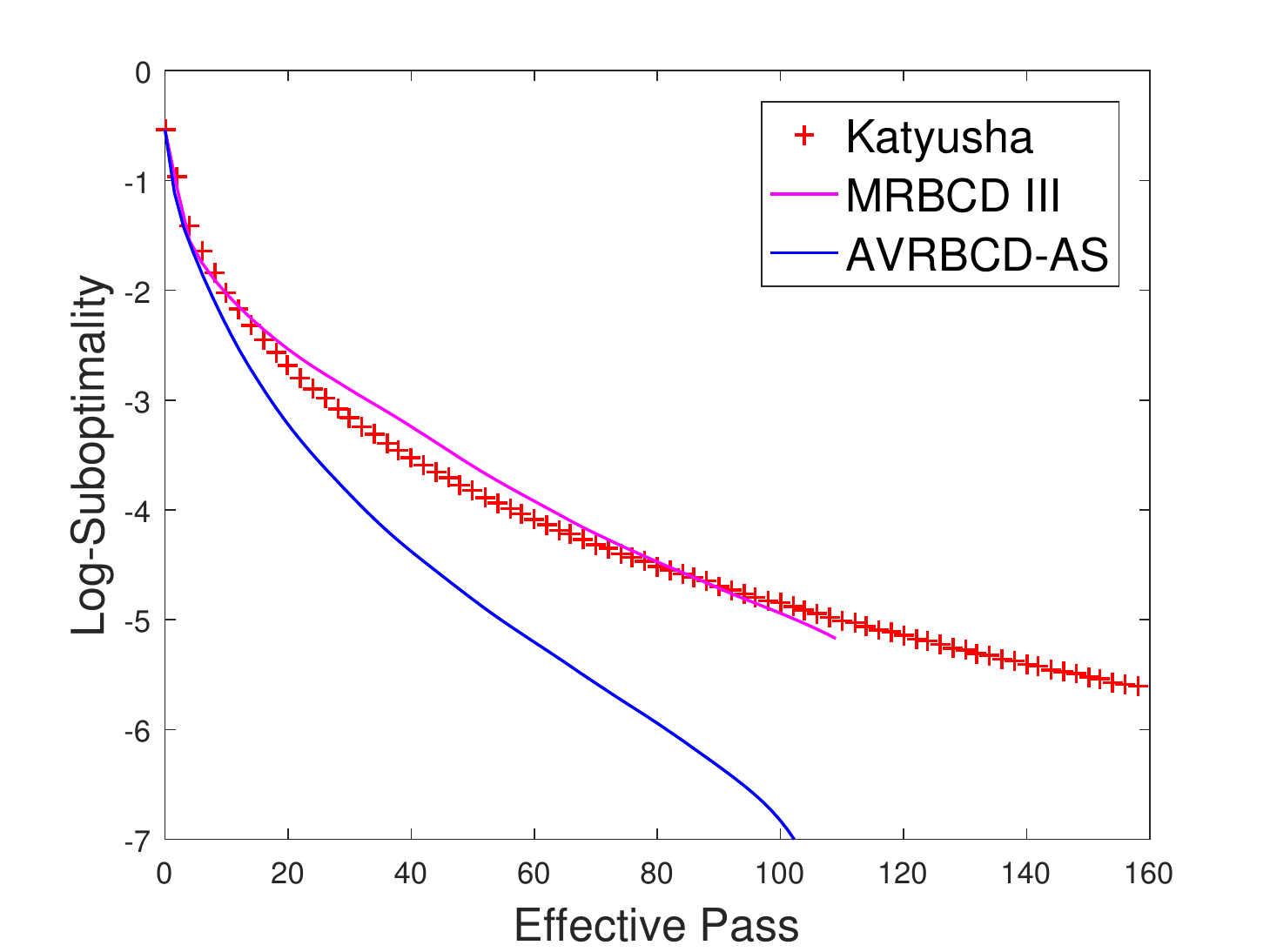}
		\end{tabular}
		\caption{$l_1l_2$-Logistic Regression. From left to right are results on rcv1 and news20.binary}
		\label{fig: l1l2 LR}
	\end{figure}
	\subsection{$l_1$-Logistic Regression}
	Three datasets are used in $l_1$-Logistic Regression, namely rcv1, real-sim, and news20.binary.
	Here, the component function is $f_i(\xB) = \log(1+\exp(-y_i\aB_i^\top\xB))$ and the regularization function is $\PB(\xB) = \lambda_1\|\xB\|_1$, where $(\aB_i, y_i)$ correspond to the feature vector and response of the $i^{th}$ sample respectively.
	In all experiments, $\lambda_1$ is set to $10^{-5}$.
	We compare the convergence rate in Figure \ref{fig: l1 LR}.
	The result shows that (i) Katyusha and AVRBCD have the similar amount of overall partial gradient evaluations, conforming to our analysis, and (ii) AVRBCD-AS has the best performance among all competitors.
	Since MRBCD III has the best performance among all non-accelerated methods, we only include MRBCD III, Katyusha, and AVRBCD III in our experiment on news20.binary.
	\subsection{$l_1l_2$-Logistic Regression}	
	Rcv1 and news20.binary are used to test the performance of our methods in $l_1l_2$-Logistic Regression.
	We set $f_i(\xB) = \log(1+\exp(-y_i\aB_i^\top\xB)) + \frac{\lambda_2}{2}\|\xB\|^2$	and $\PB(\xB) =\lambda_1\|\xB\|_1$, as suggested in {\cite{xiao2014proximal}}.
	In all experiments, $\lambda_1$ is set to $10^{-4}$ and $\lambda_2$ is set to $10^{-8}$.
	The results shows the similar phenomenon as in $l_1$-Logistic Regression, and our methods have the best computational efficiency.
\section{Conclusion}
	In this paper, we proposed an accelerated variance reduced block coordinate descent algorithm that can handle problems with large number of samples in ultra-high dimensional space.
	We compare our algorithms to state of the arts in large scale sparse learning problems, and the result is outstanding.
\small
\bibliographystyle{plain}
\bibliography{ADVRG}
\section{Appendix}
This appendix gives the proof for Lemma 1 and Theorem 2 in the AAAI paper.
	\section{Proof of Lemma 1}
	\begin{lemma}
		\label{lemma: expectation}
		\begin{align}
			\EBB_k[\|\zB_k - \xB\|^2] =~& \frac{1}{m}\|\tilde{\zB}_k - \xB\|^2 + \frac{m-1}{m}\|\zB_{k-1}-\xB\|^2 \\
			\EBB_k[\PB(\zB_k)] =~& \frac{1}{m}\PB(\tilde{\zB}_k) + \frac{m-1}{m}\PB(\zB_{k-1}) \label{eqn: regularization expectation}
		\end{align}
	\end{lemma}
	\begin{lemma}
		\label{lemma: convex combination}
		In Algorithm I, by setting $\alpha_{2, 0} = \alpha_{3, 0} =1/2B$, for $k = sm+j\geq 1$, we have
		\begin{equation}
		\xB_k = \sum_{i=0}^{s-1}\lambda_k^i \tilde{\xB}^i + \beta_j^s \tilde{\xB}^s + \sum_{l=0}^{k} \gamma_k^l \zB_l
		\end{equation}
		where $\gamma_0^0 = 1$, $\gamma_1^0 = \frac{1}{2} - \frac{1}{2B}$, $\gamma_1^1 = \frac{1}{2}$, $\beta^0_0 = 0$, $\lambda_{(s+1)m}^s = \beta_m^s$, $\lambda^i_{k+1} = \alpha_{1,s}\lambda^i_k $
		\begin{equation*}
		\gamma_{k+1}^l = 
		\begin{cases}
		\alpha_{1,s}\gamma_k^l,~&l=0,\ldots,k-1\\
		B\alpha_{1,s}\alpha_{2,s}+(1-B)\alpha_{2,s},~&l=k \\
		B\alpha_{2,s},~&l=k+1
		\end{cases}
		\end{equation*}
		and
		\begin{equation}
		\beta_{j+1}^s = \alpha_{1,s}\beta_j^s+\alpha_{3,s}.
		\end{equation}
		Additionally, we have $\sum_{i=0}^{s-1}\lambda_k^i + \beta_j^s + \sum_{l=0}^k \gamma_k^l = 1$ and each entry in this sum is non-negative for all $k\geq1$, i.e. $\xB_k$ is a convex combination of $\{\tilde{\xB}^i\}_{i=0}^s$ and $\{\zB_l\}_{l=0}^k$.
	\end{lemma}
	\begin{proof}
		When $s=0$, 
		\begin{align*}
			\xB_0 &= \zB_0 \\
			\yB_1 &= \alpha_{1, 0} \zB_0 + \alpha_{2, 0}\zB_0 + \alpha_{3, 0}\tilde{\xB}\\
			\xB_1 &= (\alpha_{1, 0} +\alpha_{2, 0})\zB_0 + B\alpha_{2, 0}(\zB_1 - \zB_0) + \alpha_{3, 0}\tilde{\xB}^0\\
					& = (\frac{1}{2} - \alpha_{3, 0})\zB_0 + \frac{1}{2}\zB_1 + \alpha_{3, 0}\tilde{\xB}^0
		\end{align*}
		which proves the initialization.
		We prove by induction.
		Assume that our formulation is correct up till the $\kappa^{th}$ iteration
		In the following iterations,
		\begin{align*}
			\yB_{\kappa+1} =~& \alpha_{1, 0} \xB_\kappa + \alpha_{2, 0}\zB_\kappa + \alpha_{3, 0}\tilde{\xB}\\
			\xB_{\kappa+1} =~& \alpha_{1, 0} \xB_\kappa + \alpha_{2, 0}\zB_\kappa + \alpha_{3, 0}\tilde{\xB} + \alpha_{2, 0}B(\zB_{\kappa+1} - \zB_\kappa)\\
			=~& \alpha_{1, 0} \sum_{l=1}^{\kappa -1} \gamma_\kappa^l \zB_l + (B\alpha_{1,s}\alpha_{2,s}+(1-B)\alpha_{2,s})\zB_\kappa \\
			&+ B\alpha_{2,s}\zB_{\kappa+1} + (\alpha_{1,0}\beta_j^0+\alpha_{3,0})\tilde{\xB}
		\end{align*}
		which gives us the result.
		When $s \geq 1$, the same induction holds except the additional $\sum_{i=0}^{s-1}\lambda_k^i \tilde{\xB}^i$. See that $\tilde{\xB}^i$ is only added after the $i^th$ epoch is done. So it should be initialized as $\lambda^i_{k+1} = \alpha_{1,s}\lambda^i_k $.
	\end{proof}
	\section{Proof of Theorem 2}
	Suppose in the $k^{th}$ iteration, function $i_k$ is sampled from all $n$ subfunctions and block $l_k$ is sampled from all $B$ blocks.
	Define $\tilde{\zB}_k$ to be the vector if all $B$ blocks are updated in the $k^{th}$ iteration, i.e. $\tilde{\zB}_k =\prox_{\eta_s\PB_l}(\zB_{k-1} - \alpha_s \vB_k)$. 
	Clearly, for all $l \in [B]$ we have 
	\begin{equation}
	[\zB_k]_l = 
	\begin{cases}
	[\tilde{\zB}_k]_l & l = l_k,\\
	[\zB_{k-1}]_l & l \neq l_k.
	\end{cases}
	\end{equation}
	\begin{theorem} \label{theorem: proximal}
		By setting $\alpha_{2, 0} = \alpha_{3, 0} = \frac{1}{2B}$ and $\alpha_{1, 0} = 1 - \alpha_{2, 0} - \alpha_{3, 0}$, we have
		\begin{equation*}
		d(\tilde{\xB}^s) \leq \frac{\alpha_{2,s}^2}{\alpha_{3,s}}(\frac{\alpha_{1,0}}{\alpha_{2,0}^2}\frac{d(\xB_0)}{m} + \frac{\alpha_{3,0}}{\alpha_{2,0}^2} d(\xB_0) + \frac{\bar{L}_0B^2}{2m}\|\xB^* - \xB_0\|^2)
		\end{equation*}
	\end{theorem}
	\begin{proof}
	From Lemma \ref{lemma: convex combination}, we have 
	\begin{equation}
		\label{eqn: convex combination}
		\PB(\xB_k) \leq \sum_{i=0}^{s-1}\lambda_k^i \PB(\tilde{\xB}^i) + \beta_k\PB(\tilde{\xB}) + \sum_{l=0}^{k} \gamma_k^l\PB(\zB_l) \defi \hat{\PB}(\xB_k).
	\end{equation}
	Using (\ref{eqn: regularization expectation}) in Lemma \ref{lemma: expectation}, we have
	\begin{align*}
		\EBB_k \hat{\PB}(\xB_k) =&  \sum_{i=0}^{s-1}\lambda_k^i \PB(\tilde{\xB}^i) + \beta_k\PB(\tilde{\xB}) + \sum_{l=0}^{k-1} \gamma_k^l\PB(\zB_l) + m\alpha_2\EBB_k \PB(\zB_k)\\
		=& \sum_{i=0}^{s-1}\lambda_k^i \PB(\tilde{\xB}^i)+ \beta_k\PB(\tilde{\xB}) + \sum_{l=0}^{k-1} \gamma_k^l\PB(\zB_l) \\
		&+ \alpha_2(m-1)\PB(\zB_{k-1})+\alpha_2\PB(\tilde{\zB}_k).
	\end{align*}
	Assume that in the $k^{th}$ iteration, the $l^{th}$ block is selected.
	Let $L_l$ be the Lipschitz smoothness parameter of function $\FB$ in the $l^{th}$ block and define $L_B = \max_l L_l$
	\begin{align*}
		&\FB(\xB_k) \\
		\leq~& \FB(\yB_k) + \langle[\nabla\FB(\yB_k)]_l, [\xB_k - \yB_k]_l\rangle + \frac{L_l}{2}\|[\xB_k - \yB_k]_l\|^2 \\
		=~& \FB(\yB_k) + \langle[\nabla\FB(\yB_k) - \vB_k]_l, [\xB_k-\yB_k]_l\rangle + \frac{L_l}{2}\|[\xB_k-\yB_k]_l\|^2 \\
		&+ \langle[\vB_k]_l, [\xB_k-\yB_k]_l\rangle\\
		\leq~& \FB(\yB_k) + \frac{L_Q}{2B\alpha_{3,s}}\|[\xB_k-\yB_k]_l\|^2 + \frac{B\alpha_{3,s}}{2L_Q}\|[\nabla \FB(\yB_k) - \vB_k]_l\|^2\\
		& + \frac{L_l}{2}\|[\xB_k-\yB_k]_l\|^2 + \langle[\vB_k]_l, [\xB_k-\yB_k]_l\rangle\\
		=~& \FB(\yB_k) + \frac{\bar{L}_s}{2}\|[\xB_k-\yB_k]_l\|^2 + \frac{B\alpha_{3,s}}{2L_Q}\|[\nabla \FB(\yB_k) - \vB_k]_l\|^2 \\
		&+ \langle[\vB_k]_l, [\xB_k-\yB_k]_l\rangle\\
		=~& \FB(\yB_k) + \frac{B\alpha_{3,s}}{2L_Q}\|[\nabla \FB(\yB_k) - \vB_k]_l\|^2 + B\alpha_{2,s}\langle[\vB_k]_l, [\tilde{\zB}_k - \zB_{k-1}]_l\rangle \\
		&+ \frac{\bar{L}_s\alpha_{2,s}^2B^2}{2}\|[\tilde{\zB}_k - \zB_{k-1}]_l\|^2
	\end{align*}
	Taking expectation with respect to $l$, we have
	\begin{align*}
		&\EBB_l\FB(\xB_k) \\
		\leq~& \FB(\yB_k) + \frac{B\alpha_{3,s}}{2L_Q}\EBB_l\|[\nabla \FB(\yB_k) - \vB_k]_l\|^2 +B\alpha_{2,s}\EBB_l\langle[\vB_k]_l, [\tilde{\zB}_k - \zB_{k-1}]_l\rangle \\
		&+ \frac{\bar{L}_s\alpha_{2,s}^2B^2}{2}\EBB_l\|[\tilde{\zB}_k - \zB_{k-1}]_l\|^2 \\
		=~ & \FB(\yB_k) + \frac{\alpha_{3,s}}{2L_Q}\|\nabla \FB(\yB_k) - \vB_k\|^2 + \alpha_{2,s}\langle\vB_k, \tilde{\zB}_k - \zB_{k-1}\rangle \\
		&+ \frac{\bar{L}_s\alpha_{2,s}^2B}{2}\|\tilde{\zB}_k - \zB_{k-1}\|^2.
	\end{align*}
	Add the regularization term and use (\ref{eqn: convex combination}).
	\begin{align*}
		&\EBB_l \FB(\xB_k) + \hat{\PB}(\xB_k) \\
		\leq~& \FB(\yB_k) + \frac{\alpha_{3,s}}{2L_Q}\|\nabla \FB(\yB_k) - \vB_k\|^2 + \alpha_{2,s}\langle\vB_k, \tilde{\zB}_k - \zB_{k-1}\rangle + \frac{\bar{L}_s\alpha_{2,s}^2B}{2}\|\tilde{\zB}_k - \zB_{k-1}\|^2 + \EBB_l\hat{\PB}(\xB_k) \\
		=~& \FB(\yB_k) + \alpha_{3,s}\{\frac{1}{2L_Q}\|\nabla \FB(\yB_k) - \vB_k\|^2 + \langle\nabla \FB(\yB_k), \tilde{\xB} - \yB_k\rangle\}  \\
		&+ \alpha_{2,s}\langle\vB_k, \tilde{\zB}_k - \zB_{k-1}\rangle + \frac{\bar{L}_s\alpha_{2,s}^2B}{2}\|\tilde{\zB}_k - \zB_{k-1}\|^2 -\alpha_{3,s} \langle\nabla \FB(\yB_k), \tilde{\xB} - \yB_k\rangle \\
		& + \sum_{i=0}^{s-1}\lambda_k^i \PB(\tilde{\xB}^i)+ \beta_k\PB(\tilde{\xB}^s) + \sum_{l=0}^{k-1} \gamma_k^l\PB(\zB_l) + \alpha_{2,s}(B-1)\PB(\zB_{k-1})+\alpha_{2,s}\PB(\tilde{\zB}_k) \\
		=~& (1 - \alpha_{2,s})\FB(\yB_k) + \alpha_{3,s}\{\frac{1}{2L_Q}\|\nabla \FB(\yB_k) - \vB_k\|^2 + \langle\nabla \FB(\yB_k), \tilde{\xB} - \yB_k\rangle\}  \\
		&+ \alpha_{2,s}(\FB(\yB_k) + \langle\vB_k, \tilde{\zB}_k - \yB_k\rangle + \frac{\bar{L}_s\alpha_{2,s}B}{2}\|\tilde{\zB}_k - \zB_{k-1}\|^2 +\PB(\tilde{\zB}_k))\\
		& + \sum_{i=0}^{s-1}\lambda_k^i \PB(\tilde{\xB}^i) + \beta_k\PB(\tilde{\xB}^s) + \sum_{l=0}^{k-1} \gamma_k^l\PB(\zB_l) + \alpha_{2,s}(B-1)\PB(\zB_{k-1})  \\
		&+ \alpha_{2,s} \langle\vB_k, \yB_k - \zB_{k-1}\rangle -\alpha_{3,s} \langle\nabla \FB(\yB_k), \tilde{\xB} - \yB_k\rangle\\
		\leq~& (1 - \alpha_{2,s})\FB(\yB_k) + \alpha_{3,s}\{\frac{1}{2L_Q}\|\nabla \FB(\yB_k) - \vB_k\|^2 + \langle\nabla \FB(\yB_k), \tilde{\xB} - \yB_k\rangle\}  \\
		&+ \alpha_{2,s}(\FB(\yB_k) + \langle\vB_k, \xB^* - \yB_k\rangle +\PB(\xB^*) + \frac{\bar{L}_s\alpha_2m}{2}(\|\xB^* - \zB_{k-1}\|^2 - \|\xB^* - \tilde{\zB}_k\|^2))\\
		& + \sum_{i=0}^{s-1}\lambda_k^i \PB(\tilde{\xB}^i) + \beta_k\PB(\tilde{\xB}^s) + \sum_{l=0}^{k-1} \gamma_k^l\PB(\zB_l) + \alpha_{2,s}(B-1)\PB(\zB_{k-1})  \\
		&+ \alpha_{2,s} \langle\vB_k, \yB_k - \zB_{k-1}\rangle -\alpha_{3,s} \langle\nabla \FB(\yB_k), \tilde{\xB} - \yB_k\rangle
	\end{align*}
	Take expectation with respect to $i_k$ and rearrange terms.
	\begin{align*}
		&\EBB_{l, i_k} \FB(\xB_k) + \hat{\PB}(\xB_k) \\
		\leq~& (1 - \alpha_{2,s} - \alpha_{3,s})\FB(\yB_k) + \alpha_{3,s}\FB(\tilde{\xB}^s)\\
		&+ \alpha_{2,s}(\FB(\yB_k) + \langle\nabla \FB(\yB_k), \xB^* - \yB_k\rangle +\PB(\xB^*) + \frac{\bar{L}_s\alpha_{2,s}B}{2}(\|\xB^* - \zB_{k-1}\|^2 - \EBB_{i_k}\|\xB^* - \tilde{\zB}_k\|^2))\\
		&+ \sum_{i=0}^{s-1}\lambda_k^i \PB(\tilde{\xB}^i) + \beta_k\PB(\tilde{\xB}) + \sum_{l=0}^{k-1} \gamma_k^l\PB(\zB_l) + \alpha_{2,s}(B-1)\PB(\zB_{k-1})\\
		&+ \alpha_{2,s} \langle\nabla \FB(\yB_k), \yB_k - \zB_{k-1}\rangle -\alpha_{3,s} \langle\nabla \FB(\yB_k), \tilde{\xB} - \yB_k\rangle\\
		\leq~& (1 - \alpha_{2,s} - \alpha_{3,s})\FB(\yB_k) + \alpha_{3,s}\FB(\tilde{\xB})\\
		&+ \alpha_{2,s}(\FB(\yB_k) + \langle\nabla \FB(\yB_k), \xB^* - \yB_k\rangle +\PB(\xB^*) + \frac{\bar{L}_s\alpha_{2,s}B}{2}(\|\xB^* - \zB_{k-1}\|^2 - \EBB_{i_k}\|\xB^* - \tilde{\zB}_k\|^2))\\
		&+ \sum_{i=0}^{s-1}\lambda_k^i \PB(\tilde{\xB}^i) + \beta_k\PB(\tilde{\xB}) + \sum_{l=0}^{k-1} \gamma_k^l\PB(\zB_l) + \alpha_2(B-1)\PB(\zB_{k-1}) + \alpha_{1,s}\langle\nabla \FB(\yB_k), \xB_{k-1} - \yB_k\rangle \\
		\leq~& \alpha_{3,s}(\FB(\tilde{\xB})+\PB(\tilde{\xB})) + \alpha_{2,s}(\FB(\xB^*) +\PB(\xB^*)) + \frac{\bar{L}_s\alpha_{2,s}^2B}{2}(\|\xB^* - \zB_{k-1}\|^2 - \EBB_{i_k}\|\xB^* - \tilde{\zB}_k\|^2)\\
		& + \alpha_{1,s}\sum_{i=0}^{s-1}\lambda_{k-1}^i \PB(\tilde{\xB}^i) + \alpha_{1,s}\beta_{k-1}\PB(\tilde{\xB}) + \alpha_{1,s}\sum_{l=0}^{k-1} \gamma_{k-1}^l\PB(\zB_l)  + \alpha_{1,s}\FB(\xB_{k-1})\\
		=~& \alpha_{3,s}(\FB(\tilde{\xB})+\PB(\tilde{\xB})) + \alpha_{2,s}(\FB(\xB^*) +\PB(\xB^*)) + \alpha_{1,s}(\FB(\xB_{k-1})+\hat{\PB}(\xB_{k-1}))\\
		&+ \frac{\bar{L}_s\alpha_{2,s}^2B^2}{2}(\|\xB^* - \zB_{k-1}\|^2 - \EBB_{l, i_k}\|\xB^* - \zB_k\|^2)\\
	\end{align*}
	Subtract $\FB(\xB^*) +\PB(\xB^*)$ from both sides and use the fact that $\alpha_{1,s} + \alpha_{2,s} + \alpha_{3,s} = 1$, we have
	\begin{align*}
		&\EBB_{l, i_k} \fB^{\hat{\PB}}(\xB_k)  - \fB^\PB(\xB^*)\\
		\leq~& \alpha_{3,s}(\fB^\PB(\tilde{\xB}) - \fB^\PB(\xB^*)) + \alpha_{1,s}(\fB^{\hat{\PB}}(\xB_{k-1})  - \fB^\PB(\xB^*))\\
		& + \frac{\bar{L}_s\alpha_{2,s}^2B^2}{2}(\|\xB^* - \zB_{k-1}\|^2 - \EBB_{k, i_k}\|\xB^* - \zB_k\|^2)
	\end{align*}
	Define $\hat{d}_k \defi \fB^{\hat{\PB}}(\xB_k)  - \fB^\PB(\xB^*)$ and $\tilde{d} \defi \fB^\PB(\tilde{\xB}) - \fB^\PB(\xB^*)$, we have
	\begin{equation}
		\frac{1}{\alpha_{2,s}^2}\EBB_{k, i_k} \hat{d}_k \leq \frac{\alpha_{1,s}}{\alpha_{2,s}^2}\hat{d}_{k-1} + \frac{\alpha_{3,s}}{\alpha_{2,s}^2} \tilde{d} + \frac{\bar{L}_sB^2}{2}(\|\xB^* - \zB_{k-1}\|^2 - \EBB_{k, i_k}\|\xB^* - \zB_k\|^2)
	\end{equation}
	Summing from $k = sm + 1$ to $(s+1)m$, we get
	\begin{equation}
		\frac{1}{\alpha_{2,s}^2} \sum_{j=1}^{m}\hat{d}_{j,s} \leq \frac{\alpha_{1,s}}{\alpha_{2,s}^2} \sum_{j=0}^{m-1} \hat{d}_{j,s} + \frac{\alpha_{3,s}}{\alpha_{2,s}^2}m\tilde{d}_s + \frac{\bar{L}_sB^2}{2}\{\|\xB^* - \zB_{0,s}\|^2 -\|\xB^* - \zB_{m,s}\|^2\}.
	\end{equation}
	By rearranging terms, we have
	\begin{equation}
		\label{eqn: I}
		\frac{\alpha_{1,s}}{\alpha_{2,s}^2}\hat{d}_{m,s} + \frac{1-\alpha_{1,s}}{\alpha_{2,s}^2} \sum_{j=1}^{m} \hat{d}_{j,s} \leq \frac{\alpha_{1,s}}{\alpha_{2,s}^2}\hat{d}_{0,s} + \frac{\alpha_{3,s}}{\alpha_{2,s}^2}m\tilde{d}_s + \frac{\bar{L}_sB^2}{2}\{\|\xB^* - \zB_{0, s}\|^2 -\|\xB^* - \zB_{m, s}\|^2\}.
	\end{equation}
	Using the fact that 
	$\frac{1 - \alpha_{1,s}}{\alpha_{2,s}^2} = \frac{\alpha_{3,s+1}}{\alpha_{2,s+1}^2}, 
		\frac{1}{\alpha_{2,s}^2} = \frac{1 - \alpha_{2,s+1}}{\alpha_{2,s+1}^2},$
	we have $\frac{\alpha_{1,s}}{\alpha_{2,s}^2} = \frac{\alpha_{1,s+1}}{\alpha_{2,s+1}^2}$.
	From the definition of $\tilde{\xB}^s$, we have
	\begin{equation}
		 m\tilde{d}_{s+1}\leq\sum_{j=1}^{m}d_{j, s}\leq\sum_{j=1}^{m}\hat{d}_{j, s},
		 \hat{d}_{0,s} = \hat{d}_{m,s-1},
		 \zB_{m, s-1} = \zB_{0, s}
	\end{equation} 
	Additionally, using $\bar{L}_{s+1}\leq\bar{L}_s$,
	we have the following inequality
	\begin{equation}
		\begin{aligned}
			&\frac{\alpha_{1,s+1}}{\alpha_{2,s+1}^2}\hat{d}_{0,s+1} + \frac{\alpha_{3,s+1}}{\alpha_{2,s+1}^2} m\tilde{d}_{s+1} + \frac{\bar{L}_{s+1}B^2}{2}\|\xB^* - \zB_{0, s+1}\|^2 \\
			\leq & \frac{\alpha_{1,s}}{\alpha_{2,s}^2}\hat{d}_{0,s} + \frac{\alpha_{3,s}}{\alpha_{2,s}^2}m \tilde{d}_s + \frac{\bar{L}_sB^2}{2}\|\xB^* - \zB_{0, s}\|^2
		\end{aligned}
	\end{equation}
	\begin{equation}
		 \tilde{d}_{s+1} \leq \frac{\alpha_{2,s+1}^2}{\alpha_{3,s+1}}(\frac{\alpha_{1,0}}{\alpha_{2,0}^2}\frac{d_0}{m} + \frac{\alpha_{3,0}}{\alpha_{2,0}^2} d_0 + \frac{\bar{L}_0B^2}{2m}\|\xB^* - \zB_0\|^2)
	\end{equation}
	\end{proof}
	The non-proximal setting is special case of Theorem 2, and can be obtain by setting $\PB(\xB) = 0$ and remove the initialization constraints $\alpha_{2, 0} = \alpha_{3, 0} =1/2B$. 
\end{document}